\documentclass{article} % For LaTeX2e
\usepackage{microtype}
\usepackage{graphicx}
\usepackage{subfigure}
\usepackage{booktabs}

\usepackage{hyperref}

\usepackage{hyperref}
\usepackage{url}
\usepackage{amssymb,amsmath}
\usepackage{algorithm}
\usepackage{hyperref}
\usepackage{graphicx}

\usepackage[accepted]{icml2018}

\usepackage{amsthm}

\newtheorem{prop}{Proposition}

\begin{document}

\twocolumn[
\icmltitle{Joint-stochastic-approximation Autoencoders with Application to Semi-supervised Learning}
\icmlsetsymbol{equal}{*}

\begin{icmlauthorlist}
\icmlauthor{Wenbo~He}{thuee}
\icmlauthor{Zhijian~Ou}{thuee}
\end{icmlauthorlist}

\icmlaffiliation{thuee}{Department of Electronic Engineering, Tsinghua University, Beijing, China}
\icmlcorrespondingauthor{Zhijian Ou}{ozj@tsinghua.edu.cn}

\vskip 0.3in
]


\printAffiliationsAndNotice{} % otherwise use the standard text.

\newcommand{\fix}{\marginpar{FIX}}
\newcommand{\new}{\marginpar{NEW}}

%\iclrfinalcopy % Uncomment for camera-ready version

%\begin{document}
%\maketitle

\begin{abstract}
Our examination of existing deep generative models (DGMs), including VAEs and GANs, reveals two problems.
First, their capability in handling discrete observations and latent codes is unsatisfactory, though there are interesting efforts.
Second, both VAEs and GANs optimize some criteria that are indirectly related to the data likelihood.
To address these problems, we formally present Joint-stochastic-approximation (JSA) autoencoders - a new family of algorithms for building deep directed generative models, with application to semi-supervised learning.
The JSA learning algorithm directly maximizes the data log-likelihood and simultaneously minimizes the inclusive KL divergence the between the posteriori and the inference model. 
We provide theoretical results and conduct a series of experiments to show its superiority such as being robust to structure mismatch between encoder and decoder, consistent handling of both discrete and continuous variables.
Particularly we empirically show that JSA autoencoders with discrete latent space achieve comparable performance to other state-of-the-art DGMs with continuous latent space in semi-supervised tasks over the widely adopted datasets - MNIST and SVHN.
To the best of our knowledge, this is the first demonstration that discrete latent variable models are successfully applied in the challenging semi-supervised tasks.

\end{abstract}

\section{Introduction}

Semi-supervised learning (SSL) considers the problem of classification when only a small subset of the observations have corresponding class labels, and aims to leverage the large amount of unlabeled data to boost the classification performance.
Several broad classes of methods for semi-supervised learning include generative models \cite{zhu09}, transductive SVM \cite{Joachims1999TransductiveIF}, co-training \cite{cotraining}, and graph-based methods (see \cite{zhu09} for more introduction).
In recent years, significant progress has been made on representation, learning and inference with Deep Generative Models (DGMs) \cite{dayan1995helmholtz,hinton2006a,Kingma2014SemiSupervisedLW,goodfellow2014generative,li2015generative,xu2016joint}, and this stimulates an explosion of interest in utilizing DGMs for semi-supervised learning.

Semi-supervised learning with DGMs usually involves blending unsupervised learning and supervised learning. 
One justification is that the unsupervised loss (e.g. the negative marginal likelihood over the unlabeled data) provides additional regularization for the supervised loss over the labeled training data \cite{zhu09,erhan2010why,larochelle2012learning}.
Therefore, successful SSL methods often develop or adapt from unsupervised learning methods for DGMs.

DGMs define distributions over a set of variables, consisting of observation variable $x$ and hidden variable (or say latent code) $h$, often organized in multiple layers.
Early forms of DGMs dated back to works on Sigmoid Belief Networks (SBNs) \cite{Saul1996}, Helmholtz machines \cite{dayan1995helmholtz}, and probabilistic autoencoders \cite{zemel1994minimum}. 
In recent years, deep generative modeling techniques has been greatly advanced by inventing new models with new learning algorithms, such as Variational Autoencoders (VAEs) \cite{Kingma2014SemiSupervisedLW}, Generative Adversarial Networks (GANs) \cite{goodfellow2014generative}, auto-regressive neural networks \cite{Larochelle2011TheNA} and so on; all are originally proposed in the context of unsupervised learning.
Two most prominent techniques are VAEs and GANs, both of which have been successfully adapted to semi-supervised learning.
These two also represent two important classes of existing techniques in model representation, learning algorithm, and adaptation to SSL.

VAEs use prescribed generative models \cite{mohamed2016learning} and variational learning, i.e. maximize the variational lower bound of the data log-likelihood w.r.t. both the generative and inference models jointly. 
Adaptation to SSL is straightforward by introducing class label $y$ as another latent variable, in addition to $h$.
Remarkably, continuous hidden variables is in dominant use with the re-parameterization trick - even when the underlying modality is inherently discrete. Using discrete hidden variables in VAEs still remains a challenge, although there are some prior efforts \cite{mnih2014neural,jang2016categorical,maddison2016the,vqvae}.
The application of VAEs to text data has been far less successful, and recently been improved in \cite{vaeccnn}.

GANs use implicit generative models \cite{mohamed2016learning} and adversarial learning, i.e. minimize a lower bound on the Jensen-Shannon (JS) divergence between the generator distribution and
data distribution, along with a discriminator \cite{nowozin2016f-gan}.
Adaptation to SSL either use the $\left(K+1\right)$-class discriminative objective \cite{imporveGAN,badGAN,bayesGAN} or still use $K$-class classifier but with various additional regularization terms \cite{springenberg2016unsupervised, NIPS2017_6997} .
Remarkably, GANs lack the ability to infer the latent variable given the observation, and this limitation has been addressed by some recent studies \cite{dumoulin2016adversarially,Donahue2016AdversarialFL}. 
Learning GANs with discrete hidden variables remains unexplored. 
The application of GANs to discrete data is also rather restricted yet with some efforts \cite{gumbelgan, yu2016seqgan,maligan, bayesGAN}.

The above examination of VAEs and GANs in the context of unsupervised and semi-supervised learning reveals two problems with existing DGM techniques for SSL. 
First, their capability in handling discrete observations and latent codes is unsatisfactory, though there are interesting progresses.
One fundamental reason is the difficulty of back-propagation of gradients through discrete random variables.
Second, although maximum likelihood (ML) has been the de-facto standard for training generative models, both VAEs and GANs optimize some bounds that are indirectly related to the data likelihood.

To address these problems, we note that recently, a new learning algorithm called Joint-stochastic-approximation (JSA) is developed for a broad class of DGMs which are characterized by pairing a generative model (decoder) with an auxiliary inference model (encoder) which approximates the posterior inference in the generative model.
The JSA learning algorithm directly maximizes the data log-likelihood and simultaneously minimizes the inclusive KL divergence the between the posteriori and the inference model. 
We call this new DGM technique by JSA autoencoders, or JAEs for short.
Inspired by the success of unsupervised learning with JAEs in \cite{xu2016joint}, we examine its adaptation to SSL in this paper, which addresses the above two problems with existing DGM techniques for SSL. 

The contributions of this work can be summarized as:

(1) We formally introduce Joint-stochastic-approximation autoencoders (JAE) - a new family of algorithms for building deep directed generative models for semi-supervised tasks, and show its theoretical consistency in the nonparametric limit.
Two distinctive features of JAEs are that they directly optimize the data log-likelihood and provide a simple, consistent and principled way to handle both discrete and continuous variables in latent and observation space.

(2) Synthetic experiments are given to help us analyze JAEs in-depth and understand their behaviors.

(3) We empirically show that JAEs with discrete latent space achieve comparable performance to other state-of-the-art DGMs with continuous latent space in semi-supervised tasks over the widely adopted datasets - MNIST and SVHN.
To the best of our knowledge, this is the first demonstration that discrete latent variable models are successfully applied in the challenging semi-supervised tasks.

\section{Related work}

In this work, we are mainly concerned with semi-supervised learning with deep generative models.
Currently most state-of-the-art SSL methods are based on DGMs.
The main idea is that generative training over unlabeled data provides regularization for finding good classifiers \cite{zhu09}.
From the perspective of regularization, virtual adversarial training (VAT) \cite{Miyato2017VirtualAT} seeks virtually adversarial samples to smooth the output distribution of the classifier, temporal ensembling \cite{laine2016temporal} and mean teacher \cite{tarvainen2017mean} maintain running averages of label predictions and model weights respectively for regularization.
These SSL methods also achieve good results. It can be seen that these SSL methods utilize different regularization from SSL with DGMs. Their combination could yield further performance improvement in practice.

SSL with DGMs often develops or adapts from unsupervised learning methods for DGMs. Recently there have emerged a bundle of DGMs, among which VAEs and GANs represent two important classes.
For fitting a generative model to data, the optimization criterion used has profound effect on the behavior of the fitted model \cite{theis2016a}. Additionally, some auxiliary model are often introduced to facilitate the optimization, e.g. the inference model in variational learning and JSA learning, the discriminator in adversarial learning.
In the following we list a few important optimization criteria (not a complete list) along with the models or learning algorithms which use them. For the sake of clarity, we omit to compare the criteria for optimizing the auxiliary models. Note that JSA learning is distinctive since it directly optimize the data log-likelihood.
\begin{itemize}
\item Maximizing the data log-likelihood, or equivalently minimizing
the Kullback-Leibler (KL) divergence between the data distribution and the generative model, used by JSA \cite{xu2016joint}
\item Maximizing the variational lower bound of the data log-likelihood, or equivalently minimizing the KL divergence between the inference model and the posteriori, used by the wake-sleep algorithm \cite{hinton1995the}, NVIL \cite{mnih2014neural}, VAEs \cite{Kingma2014SemiSupervisedLW};

\item Maximizing the importance sampling (IS) approximated lower bound of the data log-likelihood, used by reweighted wake-sleep (RWS) \cite{bornschein2014reweighted}, importance weighted autoencoders (IWAEs) \cite{burda2015importance};

\item Minimizing the JS divergence between the generator distribution and the data distribution, used by GANs.

\item Minimizing the Wasserstein distance between the generator distribution and the data distribution, used by WGAN \cite{arjovsky2017wasserstein}, WAE \cite{Tolstikhin2017WassersteinA}.
\end{itemize}

While learning under most criteria is provably consistent given infinite model capacity and data, in practice it learns very different kinds of models with different behaviors, for example, VAEs' mode covering behavior, GANs' mode missing behavior.
There are interesting efforts to design new criteria, e.g. connecting the best of GANs and VAEs \cite{makhzani2015adversarial,mescheder2017adversarial,Pu2017SymmetricVA} for better sample generation, but few evaluate the improvements over semi-supervised tasks.
Remarkably, for adaptation of GANs to SSL using the $\left(K+1\right)$-class discriminative objective, it is observed that good semi-supervised classification performance and a good generator cannot be obtained at the same time \cite{imporveGAN}; and it is further analyzed that good semi-supervised learning indeed requires a bad generator \cite{badGAN}.
Nevertheless, among various criteria, maximum likelihood is still appealing due to its nice property (consistency, statistical efficiency, and functional invariance).

For SSL to make up for the lack of labeled training data, good matching of model assumption with the structure of data is critical.
We need to handle different modalities in both observation and latent space.
The most appropriate latent space may be discrete, continuous or even mixed, depending on the structure of data.
However, learning with VAEs and GANs in discrete settings (consisting of three main cases) encounter some difficulties, lagging far behind the progress in continuous settings.

For GANs to work with discrete observations (Case I), it is difficult to propagating gradients back from the discriminator through the generated samples to the generator.
For VAEs to work with discrete latent variables (Case II), there exists basically the same difficulty of back-propagation of gradients through discrete random variables\footnote{Gradient back-propagation through continuous random variables works by using the 'reparameterization trick'. Both VAEs and GANs use this trick in continuous settings to achieve lower variance in the gradients.}.
There have some efforts to address this difficulty.
For variational learning of discrete latent variable models, the REINFORCE-like trick is employed with various variance-reduction techniques in \cite{mnih2014neural,mnih2016variational}; VQ-VAE \cite {vqvae} approximates the gradient with the straight-through estimator \cite{Bengio2013EstimatingOP}; a few studies utilize the Concrete \cite{maddison2016concrete} or Gumbel-softmax \cite{jang2016categorical} distribution as a continuous approximation to a multinomial distribution, and then use 'reparameterization trick', although the gradients of these relaxations are biased.
For learning GANs on discrete sequences, \cite{gumbelgan} resorts to the Gumbel-softmax distribution; \cite{yu2016seqgan, maligan} models the generation of the discrete sequence as a stochastic policy in reinforcement learning and perform gradient policy update.
Remarkably, JAEs do not suffer from such difficulty, since optimizing some expectation w.r.t. discrete variables is solved by stochastic approximation, as we show in the following sections.

The application of VAEs to discrete observations (Case III) has no theoretical difficulty for gradient propagation, but has been far less successful \cite{Bowman2016GeneratingSF,Miao2016NeuralVI}.
It is found that the LSTM decoder in textual VAE does not make effective use of the latent code during training. This training collapse problem may reflect structure mismatch between the encoder and decoder, and is alleviated in \cite{vaeccnn} after controlling the contextual capacity of the decoder by using dilated CNN.
Structure mismatch between the encoder and decoder causes an irreducible biased gap from the data log-likelihood for VAEs. While this mismatch also affects the statistical efficiency for JAEs, JAE learning is still consistent, since the MIS accept/reject mechanism in JAE learning will compensate for the mismatch.

\section{Method}

\subsection{Background}
Our method is an application of the stochastic approximation (SA) framework \cite{SA51}, which basically provides a mathematical framework for stochastically solving a root finding problem, which has the form of expectations being equal to zeros.
Suppose that the objective is to find the solution $\lambda^*$ of $f(\lambda) = 0$ with
\begin{equation}
\label{eq:SA}
f(\lambda) = E_{z \sim p(\cdot; \lambda) } [ F(z;\lambda) ],
\end{equation}
where $\lambda \in R^d$ is a parameter vector of dimension $d$, and $z$ is an observation from a probability distribution $p(\cdot; \lambda)$ depending on $\lambda$, and $F(z;\lambda) \in R^d $ is a function of $z$.
Given some initialization $\lambda^{(0)}$ and $z^{(0)}$, a general SA algorithm iterates as follows.
\begin{enumerate}
\item Generate $z^{(t)} \sim K_{\lambda^{(t-1)}}(z^{(t-1)},\cdot)$, a Markov transition kernel that admits $p(\cdot; \lambda^{(t-1)})$ as the invariant distribution.
\item Set $\lambda^{(t)} = \lambda^{(t-1)} + \gamma_t F(z^{(t)};\lambda^{(t-1)}) $, where $\gamma_t$ is the learning rate.
\end{enumerate}

During each SA iteration, it is possible to generate a set of multiple observations $z$ by performing the Markov transition repeatedly 
and then use the average of the corresponding values of $F(z;\lambda)$ for updating $\lambda$, which is know as SA with multiple moves \cite{Wang2017LearningTR}.
This technique can help reduce the fluctuation due to slow-mixing of Markov transitions. 
The convergence of SA has been studied under various regularity conditions, e.g. satisfying that $\sum_{t=0}^\infty \gamma_t = \infty$ and $\sum_{t=0}^\infty \gamma_t^2 < \infty$. In practice, we can set a large learning rate at the early stage of learning and decrease to $1/t$ for convergence.

\subsection{Joint-stochastic-approximation Learning}

In the following we present the JSA learning algorithm in a more general form, which was originally proposed in \cite{xu2016joint} for learning a broad class of directed generative models.

Consider a generative model\footnote{It is straightforward to develop the algorithm when the priori $p(z)$ depends on $\theta$.} $p_\theta(x,h) \triangleq p(h) p_\theta(x|h)$, consisting of observation variable $x$, hidden variables (or say latent code) $h$, and parameters $\theta$.
It is usually intractable to directly evaluate and maximize the marginal log-likelihood $log p_\theta(x)$, but it is well-known that we have
\begin{displaymath}
\frac{\partial}{\partial\theta} log p_\theta(x) = E_{p_\theta(h|x)}\left[ \frac{\partial}{\partial\theta} logp_\theta(x,h)\right] 
\end{displaymath}

We can pair the generative model $p_\theta(x,h)$ with an auxiliary inference model $q_\phi(h|x)$, parameterized by $\phi$, which approximates the posterior $p_\theta(h|x)$ in the generative model, and jointly train the two models.
This basic idea has been proposed and enhanced many times - initially by Helmholtz machines and recently by NVIL \cite{mnih2014neural}, VAEs \cite{Kingma2014SemiSupervisedLW}, RWS \cite{bornschein2014reweighted}, IWAE \cite{burda2015importance}, and so on.
The distinctive key idea of JSA learning is that 
in addition to maximizing w.r.t. $\theta$ the marginal log-likelihood, it simultaneously minimizes w.r.t. $\phi$ the \emph{inclusive} KL divergence $KL(p_\theta(h|x)||q_\phi(h|x))$ between the posteriori and the inference model, and Fortunately, we can use the SA framework to solve the optimization problem.

Suppose that data $\mathcal{D} = \left\lbrace x_1, \cdots, x_n \right\rbrace $, which consists of $n$ observations drawn from the true but unknown data distribution $p_0(x)$ with support $\mathcal{X} $.
$\tilde{p}(x) \triangleq \frac{1}{n} \sum_{k=1}^{n} \delta(x-x_n)$ denotes the empirical distribution.
Then we can formulate the maximum likelihood learning as jointly optimizing
\begin{equation}
\label{eq:JSA_unsup_obj}
\left\{
\begin{split}
& \min_{\theta} KL\left[ \tilde{p}(x) || p_\theta(x) \right] \\
& \min_{\phi} KL\left[ \tilde{p}(x) p_\theta(h|x)|| \tilde{p}(x) q_\phi(h|x) \right] \\
\end{split}
\right.
\end{equation}
By setting the gradients to zeros, the above optimization problem can be solved by finding the root for the following system of simultaneous equations:
\begin{equation}
\label{eq:JSA_unsup_gradient}
\left\{
\begin{split}
& E_{\tilde{p}(x) p_\theta(h|x)}\left[\frac{\partial}{\partial\theta} logp_\theta(x,h)\right] = 0 \\
& E_{\tilde{p}(x) p_\theta(h|x)}\left[ \frac{\partial}{\partial\phi} logq_\phi(h|x)\right] = 0 \\
\end{split}
\right.
\end{equation}

It can be shown that Eq.(\ref{eq:JSA_unsup_gradient}) exactly follows the form of Eq.(\ref{eq:SA}), so that we can apply the SA algorithm to find its root and thus solve the optimization problem Eq.(\ref{eq:JSA_unsup_obj}). 

\begin{prop}
If Eq.(\ref{eq:JSA_unsup_gradient}) is solvable, then we can apply the SA algorithm to find its root.
\end{prop}

\begin{proof}
This can be readily shown by recasting Eq.(\ref{eq:JSA_unsup_gradient}) in the form of $f(\lambda) = 0$, with $\lambda \triangleq (\theta, \phi)^T$, $z \triangleq (k,h_1,\cdots,h_n)^T$, $p(z; \lambda) \triangleq \frac{1}{n} \prod_{k=1}^{n} p_\theta(h_k|x_k)$, and 
\begin{displaymath}
F(z; \lambda) \triangleq \left( \begin{array}{c}
\frac{\partial}{\partial\theta} logp_\theta(x_k,h_k) \\
\frac{\partial}{\partial\phi} logq_\phi(h_k|x_k)
\end{array} \right),
\end{displaymath}
where $k$ denotes a uniform index variable over ${1,\cdots\,n}$.
\end{proof}

To apply the SA algorithm, we need to construct a Markov transition kernel $K_{\lambda}(z^{(t-1)},\cdot)$ that admits $p(z; \lambda)$ as the invariant distribution. There are many options. Particularly, we can use the Metropolis independence sampler (MIS), with $p(z; \lambda)$ as the target distribution. However, the proposal $q(z; \lambda)$ is defined by first drawing $k$ uniformly over ${1,\cdots\,n}$, and then only drawing $h_k \sim q_\phi(h_k|x_k)$ without changing other $h_j$ for $j \not= k$. Given current sample $z^{(t-1)}$, MIS works as follow\footnote{We update one $h_k$ at a time so that in $\frac{w(z)}{w(z^{t-1})}$, we can cancel the intractable $p_{\theta}(x_k)$ which is appeared in the importance ratio $w(z)$. In practice, we run SA with multiple moves.}:
\begin{enumerate}

\item Propose $z \sim q(z; \lambda)$,

\item Accept $z^{t}=z$ with probability
\begin{displaymath}
min\left\lbrace 1, \frac{w(z)}{w(z^{t-1})} \right\rbrace, w(z) = \frac{p_\theta(h_k|x_k)}{q_\phi(h_k|x_k)}.
\end{displaymath}
\end{enumerate}

Since the parameters of the target and auxiliary models are jointly optimized based on the SA framework, the above method is referred to as JSA learning. 
It can be seen from the above derivation that JSA learning is general, which places no constrains on the handling of discrete variables for $x$ and $h$.
In the following, we provide more comments and comparisons with existing learning techniques.

First, note that as in JSA iterations, minimizing $KL(\tilde{p}(x) p_\theta(h|x)||\tilde{p}(x) q_\phi(h|x))$ w.r.t. $\phi$ encourages the inference model to chase the posteriori, which subsequently improves the sampling efficiency of using the inference model as the proposal for sampling the posteriori.
Also the inclusive KL divergence ensures that $q_\phi(h|x)>0$ wherever $p_\theta(h|x)>0$, which makes $q_\phi(h|x)$ a valid proposal for sampling $p_\theta(h|x)$.

Second, note that adversarial learning of GANs involves finding a Nash equilibrium to a two-player non-cooperative game. Gradient descent may fail to converge, as analyzed in \cite{imporveGAN}. 
In contrast, Eq.(\ref{eq:JSA_unsup_obj}) in JSA learning is not finding a Nash equilibrium, and thus is more stable.

Third, variational learning is to optimize the variational lower bound (VLB):
\begin{displaymath}
\max_{\theta, \phi} E_{\tilde{p}(x) q_\phi(h|x)} log \left[\frac{p_\theta(x,h)}{q_\phi(h|x)} \right]
\end{displaymath}
While the gradient w.r.t. $\theta$ is well-behaved, the trouble is that the gradient w.r.t. $\phi$ is known to have high variance. To address this problem, there are a lot of efforts, as discussed in Related work.

Fourth, note that JSA learning mainly seeks ML estimates of $\theta$, with an additional optimization over $\phi$. So JSA estimator of $\theta$ enjoys the same theoretical properties as ML estimator, even if $q_\phi$ has finite capacity. Furthermore, if both $p_\theta$ and $q_\phi$ have infinite capacity, we will obtain not only the perfect generative model but also the perfect inference model. The following proposition shows the theoretical consistency of JSA learning in the nonparametric limit.
\begin{prop}
Suppose that $n \to \infty$, and $p_\theta(x,h)$ and $q_\phi(h|x)$ have infinite capacity, then we have
(i) both KL divergences in Eq.(\ref{eq:JSA_unsup_obj}) can be minimized to attain zeros. 
(ii) If both KL divergences in Eq.(\ref{eq:JSA_unsup_obj}) attain zeros at $(\theta^*, \phi^*)$, then we have $p_{\theta^*}(x) = p_0(x)$, $q_{\phi^*}(h|x) = p_{\theta^*}(h|x)$, $x \in \mathcal{X}$.
\end{prop}

\begin{proof}
By the property of the KL divergence, and $\tilde{p}(x) \to p_0(x)$ as $n \to \infty$.
\end{proof}

Finally, note that $p_\theta(x|h)$ is often termed the decoder, and $q_\phi(h|x)$ the encoder. They could be defined either with multiple stochastic hidden layers such as SBNs \cite{Saul1996}, or with multiple deterministic hidden layers such as in VAEs \cite{Kingma2014SemiSupervisedLW}. JSA could be applied in both cases, resulting in JSA autoencoders, or JAEs for short. In this paper, we define them like those in VAEs\footnote{
In this manner, the storage for saving the latent codes per training observation is much reduced, as compared to \cite{xu2016joint}.
}. Next, we examine SSL with JAEs.

\begin{algorithm}[tb]
\caption{Semi-supervised learning with JAEs}
\label{alg:SSL}
\begin{algorithmic}
\REPEAT
\STATE \underline{Monte Carlo sampling:}
Draw a unsupervised minibatch $\mathcal{U} \sim \tilde{p}(x) p_\theta(y,h|x)$ and a supervised minibatch $\mathcal{S} \sim \tilde{p}(x,y) p_\theta(h|x,y)$;
\STATE \underline{SA updating:}
Update $\theta$ by ascending:
$\frac{1}{|\mathcal{U}|} \sum_{(x,y,h) \sim \mathcal{U}}
\frac{\partial}{\partial\theta} logp_\theta(x,y,h) $
$
+\frac{1}{|\mathcal{S}|} \sum_{(x,y,h) \sim \mathcal{S}}
\frac{\partial}{\partial\theta} logp_\theta(x,y,h)
$
Update $\phi$ by ascending:
$
\frac{1}{|\mathcal{U}|} \sum_{(x,y,h) \sim \mathcal{U}}
\frac{\partial}{\partial\phi} logq_\phi(y,h|x) $
$+ \frac{1}{|\mathcal{S}|} \sum_{(x,y,h) \sim \mathcal{S}}
\frac{\partial}{\partial\phi} \left[ logq_\phi(y,h|x)
+ \alpha log q_\phi(y|x) \right] 
$
\UNTIL{convergence}
\end{algorithmic}
\end{algorithm}

\subsection{Semi-supervised Learning with JAEs}

In semi-supervised tasks, we consider the generative model $p_\theta(x,y,h) \triangleq p(y) p(h) p_\theta(x|h)$, where, with abuse of notation, the hidden variables are decomposed to the class label $y$ and the latent code $h$.
The inference model is generally denoted as $q_\phi(y,h|x)$.
Suppose that among the data $\mathcal{D} = \left\lbrace x_1, \cdots, x_n \right\rbrace $, only a small subset of the observations, for example the first $m$ observations, have class labels, $m \ll n$. Denote these labeled data as $\mathcal{L} = \left\lbrace(x_1,y_1), \cdots, (x_m,y_m) \right\rbrace $, with $\tilde{p}(x,y)$ representing the empirical distribution. 
Then we can formulate the semi-supervised learning as jointly optimizing
\begin{equation}
\label{eq:JSA_semi_obj}
\left\{
\begin{split}
& \min_{\theta} KL\left[ \tilde{p}(x) || p_\theta(x) \right] + KL\left[ \tilde{p}(x,y) || p_\theta(x,y) \right]\\
& \min_{\phi} KL\left[ \tilde{p}(x) p_\theta(y,h|x)|| \tilde{p}(x) q_\phi(y,h|x) \right] \\
& + KL\left[ \tilde{p}(x,y) p_\theta(h|x,y)|| \tilde{p}(x,y) q_\phi(h|x,y) \right] \\
& - \alpha \sum_{(x,y) \sim \mathcal{L}} log q_\phi(y|x)
\end{split}
\right.
\end{equation}
which is similar to those SSL criteria used in \cite{larochelle2012learning,Kingma2014SemiSupervisedLW}, which are defined by hybrids of generative and discriminative criteria.
By setting the gradients to zeros, the above optimization problem can be solved by finding the root for the following system of simultaneous equations:
\begin{equation}
\label{eq:JSA_semi_gradient}
\left\{
\begin{split}
& E_{\tilde{p}(x) p_\theta(y,h|x)}\left[\frac{\partial}{\partial\theta} logp_\theta(x,y,h)\right] \\ 
& + E_{\tilde{p}(x,y) p_\theta(h|x,y)}\left[\frac{\partial}{\partial\theta} logp_\theta(x,y,h)\right] = 0 \\
& E_{\tilde{p}(x) p_\theta(y,h|x)}\left[ \frac{\partial}{\partial\phi} logq_\phi(y,h|x)\right] \\
& + E_{\tilde{p}(x,y) p_\theta(h|x,y)}\left[ \frac{\partial}{\partial\phi} logq_\phi(h|x,y)\right] \\
& + \alpha \sum_{(x,y) \sim \mathcal{L}} \frac{\partial}{\partial\phi} log q_\phi(y|x) = 0
\end{split}
\right.
\end{equation}

Similarly, it can be shown that Eq.(\ref{eq:JSA_semi_gradient}) exactly follows the form of Eq.(\ref{eq:SA}), so that we can apply the SA algorithm to find its root and thus solve the optimization problem Eq.(\ref{eq:JSA_semi_obj}). 
Also we can use MIS to draw samples from the posteriori distributions $p_\theta(y,h|x),p_\theta(h|x,y)$, while using the proposals, $q_\phi(y,h|x),p_\phi(h|x,y)$, from the inference model. The SA algorithm with multiple moves for SSL is shown in Algorithm 1.

In SSL experiments in this paper, the inference network is implemented by $q_\phi(y,h|x) \triangleq q_\phi(y|x)q_\phi(h|x)$. This implementation benefits efficient posteriori inference, which is frequently executed in training. This approximation would be safe as the MIS accept/reject mechanism will compensate for the mismatch between the target distribution and the proposal distribution. This is also empirically validated in our experiments.

\section{Experiments}

To evaluate JAEs and compare with other SSL methods (mainly VAEs and GANs), we conduct a series experiments.
We will release code to reproduce our experiments.
Experimental details are provided in Appendix.
It is worthwhile to emphasize that:
\begin{itemize}
\item In addition to giving the SSL results, we also present some unsupervised learning results, because SSL with DGMs usually involves unsupervised learning. These experiments help us to understand different learning behaviors of different SSL methods, instead of competing for unsupervised performance.
\item In addition to showing the benchmarking results over the widely adopted image datasets, experiments with the synthetic datasets are useful for understanding different learning methods since we know the oracle model which creates the data.
\end{itemize} 

\begin{figure}
\vskip 0.2in
\begin{center}
\centerline{\includegraphics[width=\columnwidth]{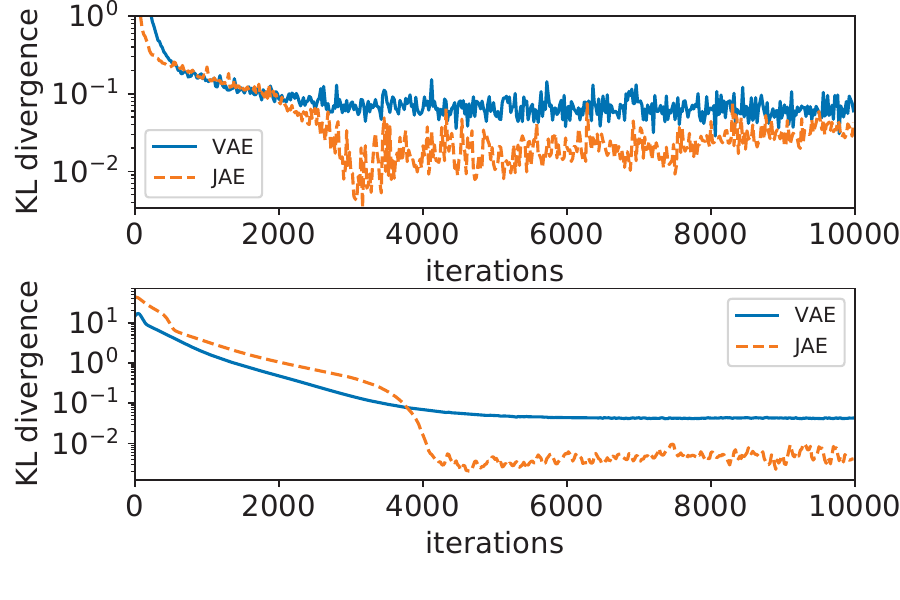}}
\caption{Results for factor analysis. \textbf{Upper:} KL divergences between $p_{\theta}(h|x)$ and $q_{\phi}(h|x)$ during training. \textbf{Lower:} KL divergences between the oracle $p_0(x)$ and the estimated $p_{\theta}(x)$ during training. }
\label{fg:KL}
\end{center}
\vskip -0.2in
\end{figure}

\subsection{Factor Analysis}

It is known that the encoders used in VAE training are usually not expressive enough to capture the true posterior distribution. They are often modeled as diagonal Gaussians whose means and variances are determined by NN transformations of the observations.
Structure mismatch between the encoder and decoder causes an irreducible gap from the data log-likelihood for VAEs. While this mismatch also affects the statistical efficiency for JAEs, JAE learning is still consistent, since the MIS accept/reject mechanism in JAE learning will compensate for the mismatch.

To illustrate this difference, we conduct an experiment over a factor analysis (FA) synthetic dataset. We create 100 3d observations with 2d latent factors as follows:

$
x\sim\mu+Ph+N(0,R), R = 0.04 \times I_3, \\
h \sim N(0,I_2),\\
\mu=(-1,0,1)^T,
P= \begin{pmatrix}

0.2 & 1 &0.5 \\

1 & 0.5 & 0.5

\end{pmatrix}^T,\\
$

We take the generative model $p_{\theta}(x,h)$ to be the above factor analysis model with unknown parameters $\theta = (\mu, P)$, so that the true posteriori $p_{\theta}(h|x)$ can be tractably calculated, which is a 2d Gaussian with correlation coefficient 0.66 in the above setting. Also the oracle $p_0(x,h)$ and assumed $p_{\theta}(x,h)$ are from the same family of models, so that we eliminate the possible bias caused by incorrect model assumption and focus on the effect of structure mismatch between the encoder and decoder.

For both VAE and JAE, the decoder $q_{\phi}(h|x)$ is implemented as a 2d Gaussian with diagonal covariance matrix. The means and variances are the four outputs from a 3-50-4 neural network, fully connected, with ReLU activations at the hidden layer and linear activations at the output layer. 
It can be seen from Figure~\ref{fg:KL} that while there is structure mismatch between the encoder and decoder for both VAE and JAE, this mismatch prevents VAE from learning the data distribution. In contrast, JAE performs much better even with a mismatched decoder.

\begin{figure}
\vskip 0.2in
\begin{center}
\centerline{\includegraphics[width=\columnwidth]{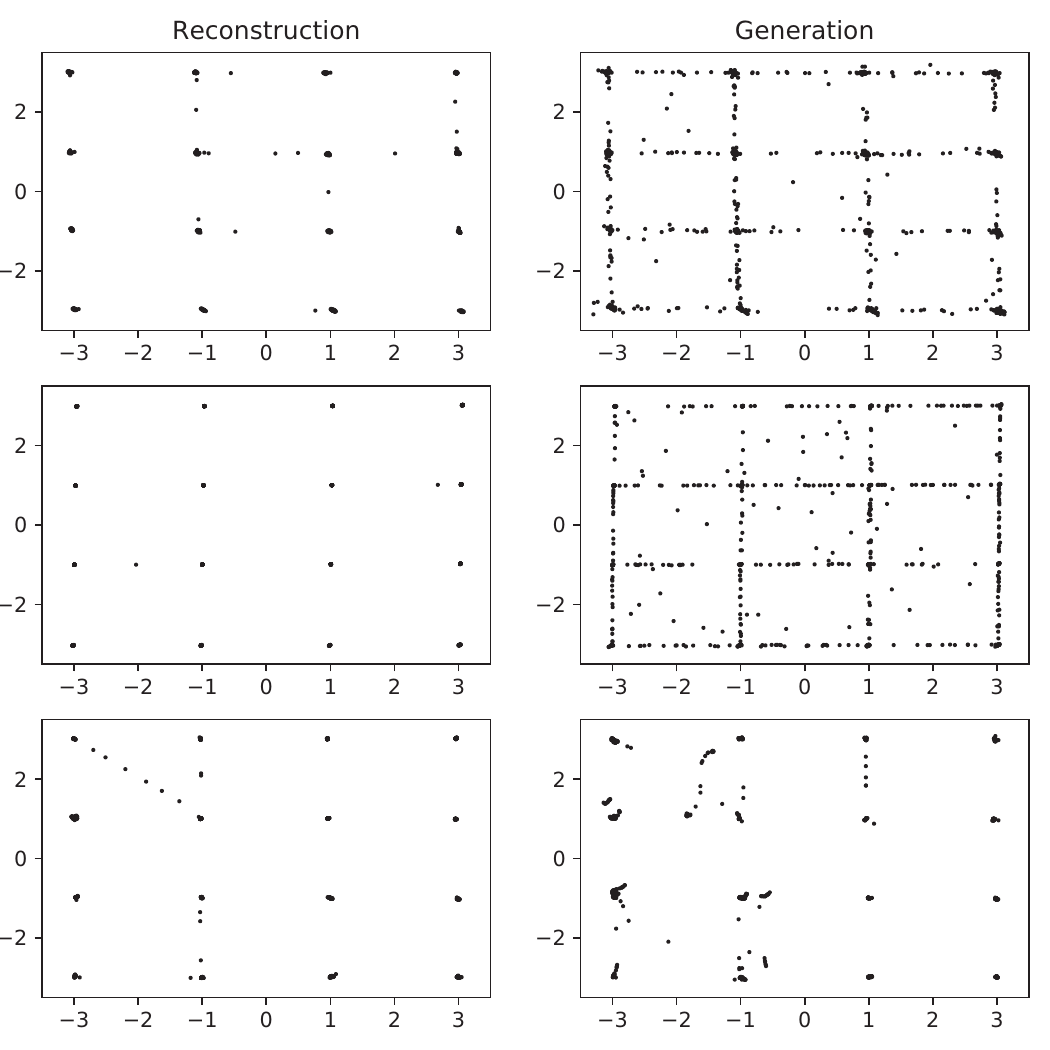}}
\caption{Comparison of a VAE with 2d Gaussian prior for latent code $h$ (row 1), a JAE with 2d Gaussian prior (row 2) and a JAE with a mixture of 4d Bernoulli and 1d Gaussian prior (row 3).}
\label{fg:mixture}
\end{center}
\vskip -0.4in
\end{figure}

\subsection{Gaussian Mixture Model}

This experiment serves two purposes.
First, it help us to understand the reconstruction behaviors of VAEs and JAEs, which also reflect the performances of the inference models.
Second, it evaluates the capabilities of VAEs and JAEs for supporting discrete latent modalities.
Note that for GANs, their reconstruction performance is still worse than autoencoders like VAEs, as evidenced in recent efforts to augment GANs' inference ability \cite{dumoulin2016adversarially}. So we mainly compare VAEs and JAEs in this experiment, and also in the previous FA experiment for the same reason.

As similarly done in \cite{dumoulin2016adversarially}, we create a synthetic dataset, created by randomly drawing 1600 data-points from a 2D Gaussian mixture model (GMM) with 16 equally-weighted, low-variance ($0.05^2$) Gaussian components laid out on a grid
This distribution exhibits lots of modes separated by large low-probability regions, which makes it a decently hard task for learning. It is also a good example of distributions, whose latent space contains both discrete and continuous modalities.

For both VAEs and JAEs, the inference network is implemented by a 3-layer fully connected neural network with 400 units per layer and ReLU activations. The generative network is also a 3-layer fully connected neural network with 200 units per layer and ReLU activations.
We compare a VAE with 2d Gaussian prior for latent code $h$, a JAE also with 2d Gaussian prior, and a JAE with a mixture of 4d Bernoulli and 1d Gaussian prior. The results are summarized in Figure~\ref{fg:mixture}. 

With 2d Gaussian prior, the JAE and VAE perform similarly for such continuous latent space, in terms of reconstruction and sample generation.
However, for this data, the most appropriate latent space is in fact a mixture of discrete and continuous modalities, instead of being purely continuous.
It can be seen that the JAE with the mixed latent code performs the best, which successfully discover not only the 16 discrete modes (represented by 4d Bernoulli) but also the variances surrounding each mode. This demonstrates that JAE can consistently handle both discrete and continuous latent codes.

\begin{figure}
\vskip 0.2in
\begin{center}
\centerline{\includegraphics[width=\columnwidth]{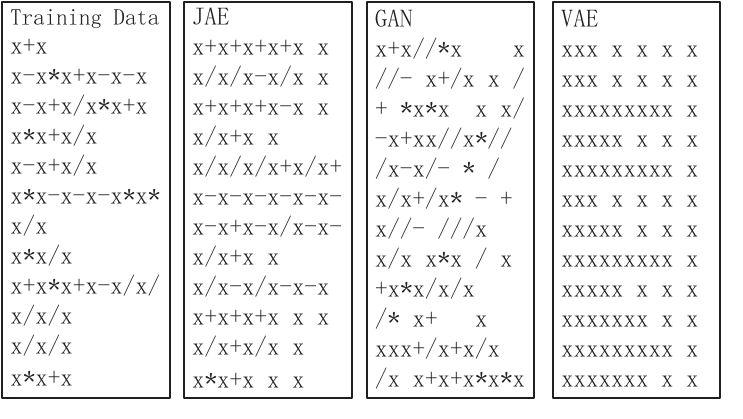}}
\caption{Column 1: Part of the training data from the context free grammar; Column 2/3/4 : data generated by JAE, GAN and VAE respectively. Both GAN and VAE use the Gumbel-softmax trick. The GAN result is copied from \cite{gumbelgan}. The temperature $\{0.1,0.01,0.001\}$ is tested for Gumbel-softmax with VAE. 
}
\label{fg:CFG}
\end{center}
\vskip -0.3in
\end{figure}

\subsection{Sequences of Discrete Elements}

In this experiment, we compare VAEs, GANs and GANs for handling discrete observations.
Particularly, we consider context free grammar (CFG) learning, as done in \cite{gumbelgan}.
The CFG is: $S \rightarrow x||S+S||S-S||S*S||S/S$, which is used to create sequences with a maximum length of 12 characters. 
We pad all sequences with less than 12 characters with spaces.
There are six characters $\{+,-,*,/,x,space\}$. 

We implement VAE and JAE for this learning task. 
Both models use the same RNN-based seq2seq architecture for encoder-decoder.
The decoder is basically a 2-layer (50-6) LSTM, representing $logp_{\theta}(x_{1...T}|h) = \sum_{t=1}^T logp_{\theta}(x_t|x_{t-1},h)$, where $x_t$ is the one-hot encoded character and $x_0=\textbf{0}$.
The prior $p(h)$ is 20d Bernoulli with $\mu=0.5$.
The encoder is basically a 2-layer (50-50) LSTM, representing
$q_{\phi}(h|x_{1...T})$. It calculates the posteriori of $h$ by a single layer neural network form last state of the LSTM. 

The sequence generation results are shown in Figure~\ref{fg:CFG}.
We can see that the generative model of JAE has learned that $x_1$ should be the character 'x', and the character 'x' and other symbols ${+,-,*,/}$ should be generated alternately. 
VAE has learned the importance of 'x' and spaces, but omits other symbols.
This result shows that JAEs are superior in learning with sequences of discrete elements.

\begin{table*}[t]
\caption{Comparison (\% error) between state-of-the-art SSL methods over MNIST using 100 labels and SVHN using 1000 labels, without data augmentation and self-ensembling.}
\vskip 0.15in
\begin{center}
\begin{small}
\begin{sc}
\begin{tabular}{lccc} % {lccc} 表示各列元素对齐方式，left-l,right-r,center-c
\hline
Methods 									&MNIST 		&SVHN 			&CIFAR-10\\ % \hline 在此行下面画一横线
\hline
Ladder network \yrcite{rasmus2015semi-supervised} 	&1.06$\pm$0.37 	&- 				&20.40$\pm$0.47\\
CatGAN \yrcite{springenberg2016unsupervised} 		&1.91$\pm$0.10 	&-  				&19.58$\pm$0.45\\ % \\ 表示重新开始一行
ImprovedGAN \yrcite{imporveGAN}					&0.93$\pm$0.06 	&8.11$\pm$1.3 		&18.63$\pm$2.32\\
ALI \yrcite{dumoulin2016adversarially} 			&- 			&7.42$\pm$0.65  	&17.99$\pm$1.62\\
BadGan \yrcite{badGAN} 						&0.80$\pm$0.01 	&4.25$\pm$0.03 		&14.41$\pm$0.30 \\
TripleGan \yrcite{chongxuan2017triple} 			&0.91$\pm$0.58 	&5.77$\pm$0.17 		&16.99$\pm$0.36\\
\hline
VAE \yrcite{Kingma2014SemiSupervisedLW} 			&3.33$\pm$0.14	&36.02$\pm0.10^*$ 	&-\\
SDGM \yrcite{maaloe2016auxiliary} 				&1.32$\pm$0.07 	&16.61$\pm0.24^*$ 	&-\\ % & 表示列的分隔线
ADGM \yrcite{maaloe2016auxiliary} 				&0.96$\pm$0.02	&$22.86^*$		 	&-\\
ST Gumbel-softmax \yrcite{jang2016categorical} 		&6.4 			&-  				&-\\
JAE+Gaussian 								&1.98$\pm$0.07 	&12.80$\pm0.32^*$  	&-\\
JAE+Bernoulli 								&1.96$\pm$0.12 	&6.22$\pm0.55^*$ 	&44.8\\
\hline
\end{tabular}
\end{sc}
\end{small}
\end{center}
\vskip -0.1in
\label{tb:acc-semi}
\end{table*}

\begin{table*}[t]
\caption{Evalution of log-likehood on binary MNIST dataset}
\vskip 0.15in
\begin{center}
\begin{small}
\begin{sc}
\begin{tabular}{lcc} % {lccc} 表示各列元素对齐方式，left-l,right-r,center-c
\hline
Methods 									&$logpx\ge$ 		&$logp(x)=$ 	\\
\hline
AVB(8-dim)								 	&-83.6$\pm$0.4 		&-91.2$\pm$0.6 	\\
AVB+AC(32-dim) 								&-79.5$\pm$0.3 		&-80.2$\pm$0.4  \\ % \\ 表示重新开始一行
VAE(32-dim)								&-87.2$\pm$0.2 		&-81.9$\pm$0.4 	\\
VAE+NF(T=80)						 		&-85.1			&-  \\
VAE+HVI(T=16)								&-88.3			&-85.5\\			
convVAE+HVI(T=16)							&-84.1			&-81.9\\
VAE+VGP(2hl)								&-81.3			&-\\
DRAW+VGP									&-79.9			&-\\
VAE+IAF									&-80.8			&-79.1\\
Auxiliary VAE(L=2)							&-83.0			&-\\
IWAE(2l,k=50)								&- 				&-82.9\\
RWS										&- 				&-86.8\\
JAE(32-dim)								&-88.0			&-80.3\\
\hline
\end{tabular}
\end{sc}
\end{small}
\end{center}
\vskip -0.1in
\label{tb:acc}
\end{table*}

\begin{figure}
\vskip 0.2in
\begin{center}
\centerline{\includegraphics[width=\columnwidth]{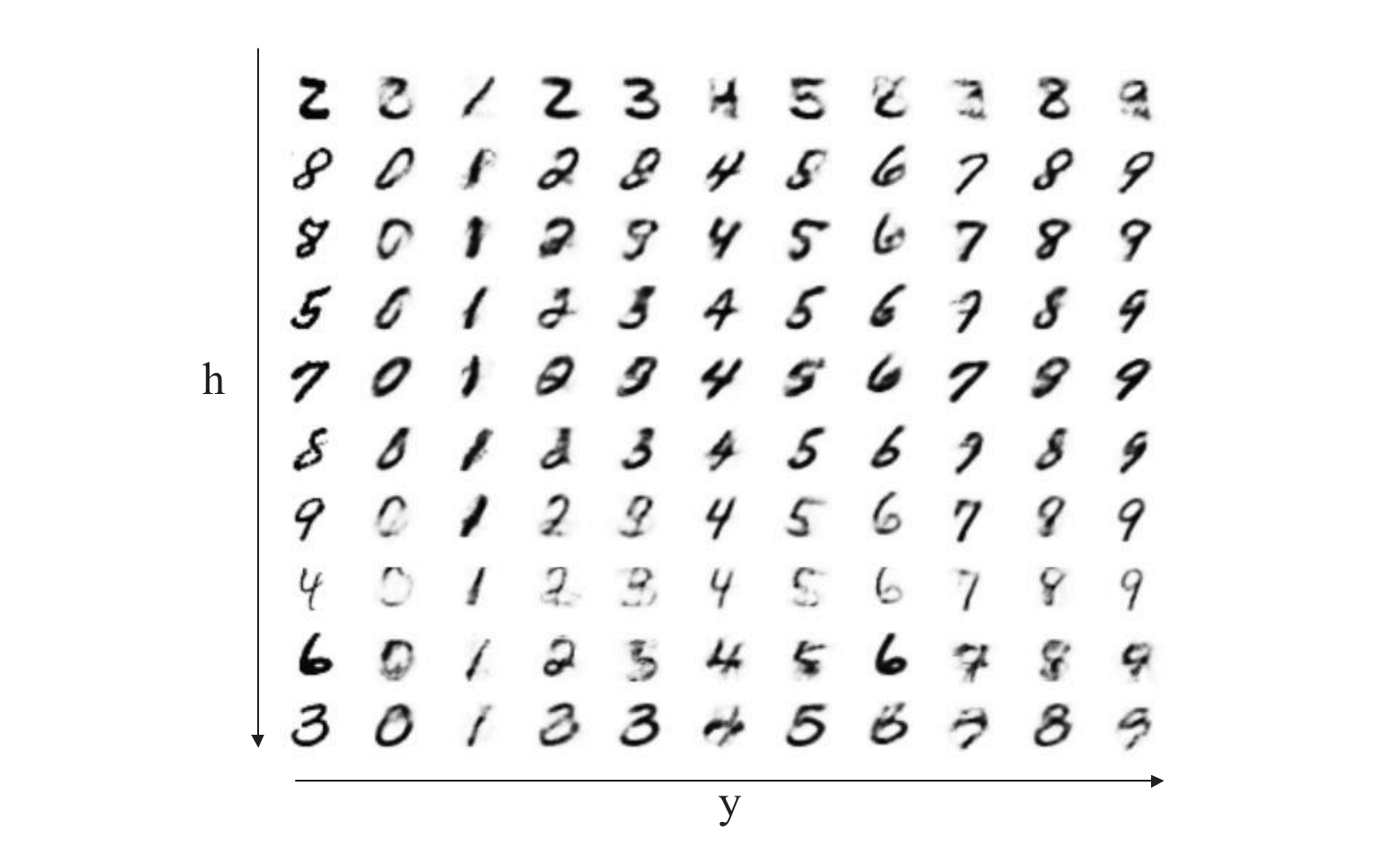}}
\caption{Conditional generation by the semi-supervised JAE over the MNIST dataset, using 60d Bernoulli prior. 
The leftmost column shows images from the test set. 
The other columns are generated by varying class label y for each column, and keeping latent h inferred from the leftmost column. 
}
\label{fg:disentengle}
\end{center}
\vskip -0.3in
\end{figure}

\subsection{Semi-supervised classification benchmarking}
We evaluate SSL performances over the widely adopted image datasets - MNIST and SVHN, and strictly follow the experimental setup in previous benchmarking studies.
MNIST contains 28x28 gray images of ten digits, consisting of 60k training samples and 10k test samples.
SVHN contains 32x32x3 RGB images of ten digits, consisting of 600k training samples (including the extra set) and 26k test samples.
We randomly choose 100 samples for MNIST and 1000 samples for SVHN as labeled data, the others in training set is unlabeled.

To demonstrate the ability of JAEs in handling discrete variables, we employ Bernoulli priors $p(h)$ for both binarized MNIST and SVHN.
The JAE models run 10 times independently for each dataset with randomly selected labeled data.
Table~\ref{tb:acc-semi} compares the JAE results with state-of-the-art SSL methods.
It can be seen that JAEs with discrete latent space achieve comparable SSL performance to other state-of-the-art DGMs with continuous latent space.
When compare with the VAE, which uses Gumbel-softmax to handle categorical class but still uses Gaussian latent code \cite{jang2016categorical}, JAE performs much better on binarized MNIST data.
In addition to superior SSL performance, we also show the ability of semi-supervised JAE to disentangle classes and styles.
These resembles the two forms of analogical reasoning with VAEs \cite{Kingma2014SemiSupervisedLW}, but we show them with discrete latent space.
Figure~\ref{fg:disentengle} shows the results by fixing the latent codes and then varying the class labels.
Figure~\ref{fg:manifold} shows the results by fixing the class label and then varying the latent codes over a range of values.

For SVHN, we find that it is beneficial to pre-process RGB images to gray images, and the error rate is around 6.80\% on RGB images. For SVHN, we additionally utilize the discriminative confidence loss to regularize the JAE's inference network, by minimizing $\sum_{x\sim \tilde{p}(x)} \sum_{y} -q_{\phi}(y|x)logq_{\phi}(y|x)$. 

\begin{figure}
\vskip 0.2in
\begin{center}
\centerline{\includegraphics[width=\columnwidth]{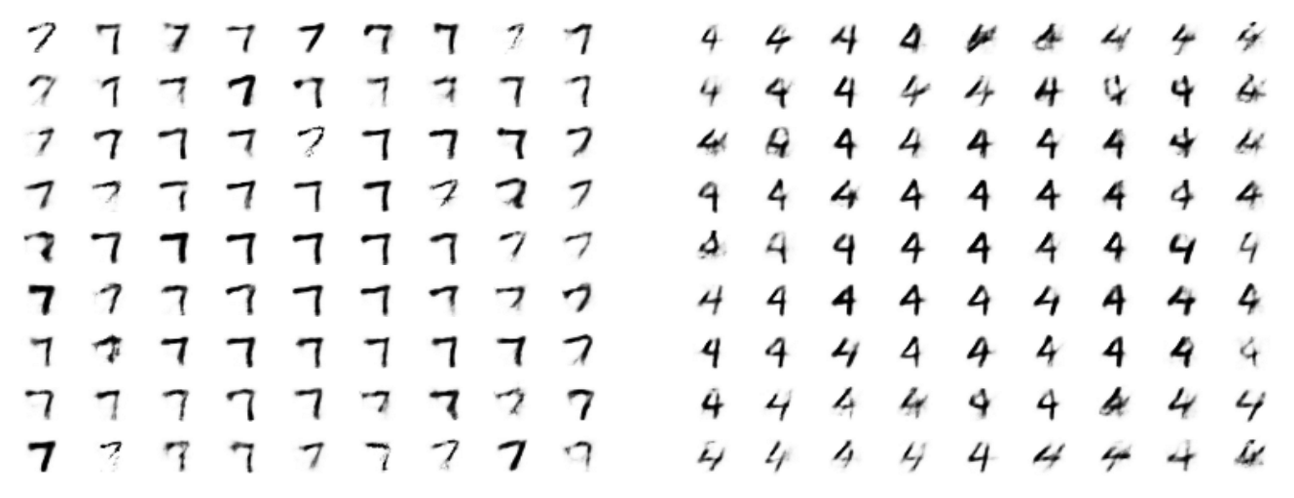}}
\caption{Class-conditional traversal in the discrete latent space. The center images in the two pictures are the reconstructions. Surrounding images are generated with several units of the latent codes flipped randomly. The number of flipped units follows the board distance to the center.}
\label{fg:manifold}
\end{center}
\vskip -0.3in
\end{figure}

\section{Conclusion}

In this paper, we formally present JSA autoencoders - a new family of algorithms for building deep directed generative models, with application to semi-supervised learning.
We provide theoretical results and conduct a series of experiments to show its superiority such as being robust to structure mismatch between encoder and decoder, consistent handling of both discrete and continuous variables.
Particularly we empirically show that JSA autoencoders with discrete latent space achieve comparable performance to other state-of-the-art DGMs with continuous latent space in semi-supervised tasks over the widely adopted datasets - MNIST and SVHN. 
In addition to variational learning and adversarial learning, JSA learning provides another tool in the machine learning toolbox, which deserves more developments.

\bibliography{JAE}

\begin{thebibliography}{53}
\providecommand{\natexlab}[1]{#1}
\providecommand{\url}[1]{\texttt{#1}}
\expandafter\ifx\csname urlstyle\endcsname\relax
  \providecommand{\doi}[1]{doi: #1}\else
  \providecommand{\doi}{doi: \begingroup \urlstyle{rm}\Url}\fi

\bibitem[Arjovsky et~al.(2017)Arjovsky, Chintala, and Bottou]{arjovsky2017wasserstein}
Arjovsky, M., Chintala, S., and Bottou, L.
\newblock Wasserstein gan.
\newblock \emph{arXiv preprint arXiv:1701.07875}, 2017.

\bibitem[Bengio et~al.(2013)Bengio, L{\'e}onard, and Courville]{Bengio2013EstimatingOP}
Bengio, Y., L{\'e}onard, N., and Courville, A.~C.
\newblock Estimating or propagating gradients through stochastic neurons for conditional computation.
\newblock \emph{CoRR}, abs/1308.3432, 2013.

\bibitem[Blum \& Mitchell(1998)Blum and Mitchell]{cotraining}
Blum, A. and Mitchell, T.~M.
\newblock Combining labeled and unlabeled data with co-training.
\newblock pp.\  92--100, 1998.

\bibitem[Bornschein \& Bengio(2014)Bornschein and Bengio]{bornschein2014reweighted}
Bornschein, J. and Bengio, Y.
\newblock Reweighted wake-sleep.
\newblock \emph{arXiv preprint arXiv:1406.2751}, 2014.

\bibitem[Bowman et~al.(2016)Bowman, Vilnis, Vinyals, Dai, J{\'o}zefowicz, and Bengio]{Bowman2016GeneratingSF}
Bowman, S.~R., Vilnis, L., Vinyals, O., Dai, A.~M., J{\'o}zefowicz, R., and Bengio, S.
\newblock Generating sentences from a continuous space.
\newblock In \emph{CoNLL}, 2016.

\bibitem[Burda et~al.(2015)Burda, Grosse, and Salakhutdinov]{burda2015importance}
Burda, Y., Grosse, R., and Salakhutdinov, R.
\newblock Importance weighted autoencoders.
\newblock \emph{arXiv preprint arXiv:1509.00519}, 2015.

\bibitem[Che et~al.(2017)Che, Li, Zhang, Hjelm, Li, Song, and Bengio]{maligan}
Che, T., Li, Y., Zhang, R., Hjelm, R.~D., Li, W., Song, Y., and Bengio, Y.
\newblock Maximum-likelihood augmented discrete generative adversarial networks.
\newblock \emph{arXiv: Artificial Intelligence}, 2017.

\bibitem[Chongxuan et~al.(2017)Chongxuan, Xu, Zhu, and Zhang]{chongxuan2017triple}
Chongxuan, L.~I., Xu, K., Zhu, J., and Zhang, B.
\newblock Triple generative adversarial nets.
\newblock \emph{neural information processing systems}, pp.\  4088--4098, 2017.

\bibitem[Dai et~al.(2017)Dai, Yang, Yang, Cohen, and Salakhutdinov]{badGAN}
Dai, Z., Yang, Z., Yang, F., Cohen, W.~W., and Salakhutdinov, R.
\newblock Good semi-supervised learning that requires a bad gan.
\newblock \emph{arXiv: Learning}, 2017.

\bibitem[Dayan et~al.(1995)Dayan, Hinton, Neal, and Zemel]{dayan1995helmholtz}
Dayan, P., Hinton, G.~E., Neal, R.~M., and Zemel, R.~S.
\newblock The helmholtz machine.
\newblock \emph{Neural computation}, 7\penalty0 (5):\penalty0 889--904, 1995.

\bibitem[Donahue et~al.(2016)Donahue, Kr{\"a}henb{\"u}hl, and Darrell]{Donahue2016AdversarialFL}
Donahue, J., Kr{\"a}henb{\"u}hl, P., and Darrell, T.
\newblock Adversarial feature learning.
\newblock \emph{CoRR}, abs/1605.09782, 2016.

\bibitem[Dumoulin et~al.(2016)Dumoulin, Belghazi, Poole, Lamb, Arjovsky, Mastropietro, and Courville]{dumoulin2016adversarially}
Dumoulin, V., Belghazi, I., Poole, B., Lamb, A., Arjovsky, M., Mastropietro, O., and Courville, A.
\newblock Adversarially learned inference.
\newblock \emph{arXiv preprint arXiv:1606.00704}, 2016.

\bibitem[Erhan et~al.(2010)Erhan, Bengio, Courville, Manzagol, Vincent, and Bengio]{erhan2010why}
Erhan, D., Bengio, Y., Courville, A.~C., Manzagol, P., Vincent, P., and Bengio, S.
\newblock Why does unsupervised pre-training help deep learning?
\newblock \emph{Journal of Machine Learning Research}, 11:\penalty0 625--660, 2010.

\bibitem[Goodfellow et~al.(2014)Goodfellow, Pougetabadie, Mirza, Xu, Wardefarley, Ozair, Courville, and Bengio]{goodfellow2014generative}
Goodfellow, I.~J., Pougetabadie, J., Mirza, M., Xu, B., Wardefarley, D., Ozair, S., Courville, A.~C., and Bengio, Y.
\newblock Generative adversarial networks.
\newblock \emph{arXiv: Machine Learning}, 2014.

\bibitem[Hinton et~al.(1995)Hinton, Dayan, Frey, and Neal]{hinton1995the}
Hinton, G.~E., Dayan, P., Frey, B.~J., and Neal, R.~M.
\newblock The "wake-sleep" algorithm for unsupervised neural networks.
\newblock \emph{Science}, 268\penalty0 (5214):\penalty0 1158--1161, 1995.

\bibitem[Hinton et~al.(2006)Hinton, Osindero, and Teh]{hinton2006a}
Hinton, G.~E., Osindero, S., and Teh, Y.~W.
\newblock A fast learning algorithm for deep belief nets.
\newblock \emph{Neural Computation}, 18\penalty0 (7):\penalty0 1527--1554, 2006.

\bibitem[Jang et~al.(2016)Jang, Gu, and Poole]{jang2016categorical}
Jang, E., Gu, S., and Poole, B.
\newblock Categorical reparameterization with gumbel-softmax.
\newblock \emph{arXiv: Machine Learning}, 2016.

\bibitem[Joachims(1999)]{Joachims1999TransductiveIF}
Joachims, T.
\newblock Transductive inference for text classification using support vector machines.
\newblock In \emph{ICML}, 1999.

\bibitem[Kingma et~al.(2014)Kingma, Rezende, Mohamed, and Welling]{Kingma2014SemiSupervisedLW}
Kingma, D.~P., Rezende, D.~J., Mohamed, S., and Welling, M.
\newblock Semi-supervised learning with deep generative models.
\newblock In \emph{NIPS}, 2014.

\bibitem[Kusner \& Hern{\'a}ndez-Lobato(2016)Kusner and Hern{\'a}ndez-Lobato]{gumbelgan}
Kusner, M.~J. and Hern{\'a}ndez-Lobato, J.~M.
\newblock Gans for sequences of discrete elements with the gumbel-softmax distribution.
\newblock \emph{CoRR}, abs/1611.04051, 2016.

\bibitem[Laine \& Aila(2016)Laine and Aila]{laine2016temporal}
Laine, S. and Aila, T.
\newblock Temporal ensembling for semi-supervised learning.
\newblock \emph{arXiv: Neural and Evolutionary Computing}, 2016.

\bibitem[Larochelle \& Murray(2011)Larochelle and Murray]{Larochelle2011TheNA}
Larochelle, H. and Murray, I.
\newblock The neural autoregressive distribution estimator.
\newblock In \emph{AISTATS}, 2011.

\bibitem[Larochelle et~al.(2012)Larochelle, Mandel, Pascanu, and Bengio]{larochelle2012learning}
Larochelle, H., Mandel, M.~I., Pascanu, R., and Bengio, Y.
\newblock Learning algorithms for the classification restricted boltzmann machine.
\newblock \emph{Journal of Machine Learning Research}, 13\penalty0 (1):\penalty0 643--669, 2012.

\bibitem[Li et~al.(2017)Li, Xu, Zhu, and Zhang]{NIPS2017_6997}
Li, C., Xu, T., Zhu, J., and Zhang, B.
\newblock Triple generative adversarial nets.
\newblock In \emph{Advances in Neural Information Processing Systems 30}, pp.\  4091--4101. Curran Associates, Inc., 2017.

\bibitem[Li et~al.(2015)Li, Swersky, and Zemel]{li2015generative}
Li, Y., Swersky, K., and Zemel, R.~S.
\newblock Generative moment matching networks.
\newblock \emph{international conference on machine learning}, pp.\  1718--1727, 2015.

\bibitem[Maaloe et~al.(2016)Maaloe, Sonderby, Sonderby, and Winther]{maaloe2016auxiliary}
Maaloe, L., Sonderby, C.~K., Sonderby, S.~K., and Winther, O.
\newblock Auxiliary deep generative models.
\newblock \emph{international conference on machine learning}, pp.\  1445--1453, 2016.

\bibitem[Maddison et~al.(2016{\natexlab{a}})Maddison, Mnih, and Teh]{maddison2016concrete}
Maddison, C.~J., Mnih, A., and Teh, Y.~W.
\newblock The concrete distribution: A continuous relaxation of discrete random variables.
\newblock \emph{arXiv preprint arXiv:1611.00712}, 2016{\natexlab{a}}.

\bibitem[Maddison et~al.(2016{\natexlab{b}})Maddison, Mnih, and Teh]{maddison2016the}
Maddison, C.~J., Mnih, A., and Teh, Y.~W.
\newblock The concrete distribution: A continuous relaxation of discrete random variables.
\newblock \emph{arXiv: Learning}, 2016{\natexlab{b}}.

\bibitem[Makhzani et~al.(2015)Makhzani, Shlens, Jaitly, Goodfellow, and Frey]{makhzani2015adversarial}
Makhzani, A., Shlens, J., Jaitly, N., Goodfellow, I., and Frey, B.
\newblock Adversarial autoencoders.
\newblock \emph{arXiv preprint arXiv:1511.05644}, 2015.

\bibitem[Mescheder et~al.(2017)Mescheder, Nowozin, and Geiger]{mescheder2017adversarial}
Mescheder, L., Nowozin, S., and Geiger, A.
\newblock Adversarial variational bayes: Unifying variational autoencoders and generative adversarial networks.
\newblock \emph{arXiv preprint arXiv:1701.04722}, 2017.

\bibitem[Miao et~al.(2016)Miao, Yu, and Blunsom]{Miao2016NeuralVI}
Miao, Y., Yu, L., and Blunsom, P.
\newblock Neural variational inference for text processing.
\newblock In \emph{ICML}, 2016.

\bibitem[Miyato et~al.(2017)Miyato, ichi Maeda, Koyama, and Ishii]{Miyato2017VirtualAT}
Miyato, T., ichi Maeda, S., Koyama, M., and Ishii, S.
\newblock Virtual adversarial training: a regularization method for supervised and semi-supervised learning.
\newblock \emph{CoRR}, abs/1704.03976, 2017.

\bibitem[Mnih \& Gregor(2014)Mnih and Gregor]{mnih2014neural}
Mnih, A. and Gregor, K.
\newblock Neural variational inference and learning in belief networks.
\newblock \emph{international conference on machine learning}, pp.\  1791--1799, 2014.

\bibitem[Mnih \& Rezende(2016)Mnih and Rezende]{mnih2016variational}
Mnih, A. and Rezende, D.~J.
\newblock Variational inference for monte carlo objectives.
\newblock \emph{international conference on machine learning}, pp.\  2188--2196, 2016.

\bibitem[Mohamed \& Lakshminarayanan(2016)Mohamed and Lakshminarayanan]{mohamed2016learning}
Mohamed, S. and Lakshminarayanan, B.
\newblock Learning in implicit generative models.
\newblock \emph{arXiv: Machine Learning}, 2016.

\bibitem[Nowozin et~al.(2016)Nowozin, Cseke, and Tomioka]{nowozin2016f-gan}
Nowozin, S., Cseke, B., and Tomioka, R.
\newblock f-gan: Training generative neural samplers using variational divergence minimization.
\newblock \emph{neural information processing systems}, pp.\  271--279, 2016.

\bibitem[Pu et~al.(2017)Pu, Chen, Dai, Wang, Li, and Carin]{Pu2017SymmetricVA}
Pu, Y., Chen, L., Dai, S., Wang, W., Li, C., and Carin, L.
\newblock Symmetric variational autoencoder and connections to adversarial learning.
\newblock \emph{CoRR}, abs/1709.01846, 2017.

\bibitem[Rasmus et~al.(2015)Rasmus, Valpola, Honkala, Berglund, and Raiko]{rasmus2015semi-supervised}
Rasmus, A., Valpola, H., Honkala, M., Berglund, M., and Raiko, T.
\newblock Semi-supervised learning with ladder networks.
\newblock \emph{neural information processing systems}, pp.\  3546--3554, 2015.

\bibitem[Robbins \& Monro(1951)Robbins and Monro]{SA51}
Robbins, H. and Monro, S.
\newblock A stochastic approximation method.
\newblock \emph{The Annals of Mathematical Statistics}, pp.\  400--407, 1951.

\bibitem[Saatchi \& Wilson(2017)Saatchi and Wilson]{bayesGAN}
Saatchi, Y. and Wilson, A.~G.
\newblock Bayesian gan.
\newblock \emph{arXiv: Learning}, 2017.

\bibitem[Salimans et~al.(2016)Salimans, Goodfellow, Zaremba, Cheung, and Alec~Radford]{imporveGAN}
Salimans, T., Goodfellow, I., Zaremba, W., Cheung, V., and Alec~Radford, X.~C.
\newblock Improved techniques for training gans.
\newblock \emph{arXiv: Learning}, 2016.

\bibitem[Saul et~al.(1996)Saul, Jaakkola, and Jordan]{Saul1996}
Saul, L.~K., Jaakkola, T., and Jordan, M.~I.
\newblock Mean field theory for sigmoid belief networks.
\newblock \emph{Journal of artificial intelligence research}, 4\penalty0 (1):\penalty0 61--76, 1996.

\bibitem[Springenberg(2016)]{springenberg2016unsupervised}
Springenberg, J.~T.
\newblock Unsupervised and semi-supervised learning with categorical generative adversarial networks.
\newblock \emph{international conference on learning representations}, 2016.

\bibitem[Tarvainen \& Valpola(2017)Tarvainen and Valpola]{tarvainen2017mean}
Tarvainen, A. and Valpola, H.
\newblock Mean teachers are better role models: Weight-averaged consistency targets improve semi-supervised deep learning results.
\newblock In \emph{Advances in neural information processing systems}, pp.\  1195--1204, 2017.

\bibitem[Theis et~al.(2016)Theis, Den~Oord, and Bethge]{theis2016a}
Theis, L., Den~Oord, A.~V., and Bethge, M.
\newblock A note on the evaluation of generative models.
\newblock \emph{international conference on learning representations}, 2016.

\bibitem[Tolstikhin et~al.(2017)Tolstikhin, Bousquet, Gelly, and Sch{\"o}lkopf]{Tolstikhin2017WassersteinA}
Tolstikhin, I.~O., Bousquet, O., Gelly, S., and Sch{\"o}lkopf, B.
\newblock Wasserstein auto-encoders.
\newblock \emph{CoRR}, abs/1711.01558, 2017.

\bibitem[van~den Oord et~al.(2017)van~den Oord, Vinyals, and Kavukcuoglu]{vqvae}
van~den Oord, A., Vinyals, O., and Kavukcuoglu, K.
\newblock Neural discrete representation learning.
\newblock In \emph{NIPS}, 2017.

\bibitem[Wang et~al.(2017)Wang, Ou, and Tan]{Wang2017LearningTR}
Wang, B., Ou, Z., and Tan, Z.
\newblock Learning trans-dimensional random fields with applications to language modeling.
\newblock \emph{IEEE transactions on pattern analysis and machine intelligence}, 2017.

\bibitem[Xu \& Ou(2016)Xu and Ou]{xu2016joint}
Xu, H. and Ou, Z.
\newblock Joint stochastic approximation learning of helmholtz machines.
\newblock \emph{arXiv preprint arXiv:1603.06170}, 2016.

\bibitem[Yang et~al.(2017)Yang, Hu, Salakhutdinov, and Berg-Kirkpatrick]{vaeccnn}
Yang, Z., Hu, Z., Salakhutdinov, R., and Berg-Kirkpatrick, T.
\newblock Improved variational autoencoders for text modeling using dilated convolutions.
\newblock \emph{arXiv: Neural and Evolutionary Computing}, 2017.

\bibitem[Yu et~al.(2016)Yu, Zhang, Wang, and Yu]{yu2016seqgan}
Yu, L., Zhang, W., Wang, J., and Yu, Y.
\newblock Seqgan: Sequence generative adversarial nets with policy gradient.
\newblock \emph{national conference on artificial intelligence}, pp.\  2852--2858, 2016.

\bibitem[Zemel(1994)]{zemel1994minimum}
Zemel, R.~S.
\newblock \emph{A minimum description length framework for unsupervised learning.}
\newblock University of Toronto, 1994.

\bibitem[Zhu(2005)]{zhu09}
Zhu, X.
\newblock Semi-supervised learning literature survey.
\newblock 2005.

\end{thebibliography}
\bibliographystyle{icml2018}

\onecolumn

\appendix
\section{Visualization}

Figure~\ref{fg:svhn_rec} demonstrates the reconstruction result of SVHN data in test set trained by semi-JAE with 220d Bernoulli $p(h)$. 

Figure~\ref{fg:tsne} demonstrates the manifold of MNIST dataset, images are represented as a 60d $p(h|x)$ probabilistic vector trained by unsupervised JAE with Bernoulli prior. It suggests that JAE can extract class features and benefit semi-supervised learning.

\section{Training Details}

\subsection{Factor Analysis Dataset}
For the synthetic factor analysis dataset, we used neural networks for the inference network with one hidden layer, containing 50 rectified linear units. The output of the inference network is a 4d vector, involve 2d for mean values and 2d for variances of the Gaussian $q(h|x)$. The generative network $x=f(h)$ is a factor generate process that $x=\hat{P}h+\mu$. Architectures for JAE and VAE is the same. Parameters are initializer by the Xavier initialization. We optimizes JAE and VAE by Adam($\alpha=10^{-2},\beta_1=0.9,\beta_2=0.999$) with learning rate 0.01. The size of the dataset is 100, and we run 10000 iterations with full batch data on it.

\subsection{Almost Discrete Gaussian Mixture Dataset}
Inference networks and generative networks we used are fully connected neural networks for the almost discrete Gaussian mixture dataset. Inference networks for VAE and JSA are networks with 2 hidden layers with shape 400-400 and the ReLU activation. Generative networks are networks 2 hidden layers with shape 200-200 and the ReLU activation. Parameters are initializer by the Xavier initialization. The variance of Gaussian noise attached to the output of JAE's inference network is selected to be 0.05, and is decreased to 0.01 after 1000 iterations. We use the Adam optimizer ($\alpha=0.05,\beta_1=0.9,\beta_2=0.999$). The size of the dataset is 1600, batch size is 100, and we run 10000 iterations on it.

\subsection{Sequences of Discrete Elements Dataset}
Architecture of the models for the discrete sequence dataset is shown in Figure~\ref{fg:graph_lstm}. Numbers on the picture denotes widths of layers. We show 4 time steps but LSTM networks run 12 time steps for every sequence. Squares denote LSTM layer and the circle denotes the fully connected layer. All layers without annotation explicitly use tanh activation. Parameters are initializer by the Xavier initialization. SGD optimizers with learning rate 0.0005 are used to train networks. The prior $p(h)$ for JAE and VAE is selected to be 20d Bernoulli distribution with $\mu=0.5$. The size of the dataset is 5000, batch size is 100, and we run 10000 iterations on it.

\subsection{MNIST}
We leverage 60d Bernoulli prior of the hidden variables for MNIST. Figure~\ref{fg:graph_mnist} demonstrates networks for MNIST. Dense denotes fully connected layer and noise denotes the addition of gauss ion noise with variance $0.3^2$. All layers without annotation explicitly use ReLU activation. Parameters are initializer by the Xavier initialization. The optimizer is SGD with learning rate 0.001. The MIS process work without cache and restart for every datapoint with 10-step warm up. After 100 epochs, the learning rate decays to 0.995 times it for every epoch. We use 100 randomly selected images as labeled data, and use the other 50k images as unlabeled data. The accuracy is evaluated on the test set with 10k images. The algorithm runs $1.2*10^5$ iterations, and for every iteration parameters are updated by 100 labeled data and 100 unlabeled data. For first 500 iterations, only supervised learning is executed.

\subsection{SVHN}
We leverage 220d Bernoulli prior of the hidden variables for SVHN. Figure~\ref{fg:graph_svhn} demonstrates networks for SVHN. BN denotes batch normalization. Dense denotes fully connected layer. And the keep probability for dropout we used is 0.8. All layers without annotation explicitly use ReLU activation. The variance of the Gaussian noise attached to the output of JAE's inference network is selected to be 0.1. Parameters are initializer by the Xavier initialization and biases are set to be zero. Regulation are added into the criterion. Minimizing $H(p(y|x))$ for unlabeled data and the pseudo loss $KL(\sum_hp(x,y,h)||q(z|x)p(x))$ improves the performance of the network. The optimizer is SGD with learning rate 0.00005. After $10^5$ iterations, the learning rate decays to 0.995 times it for every 100 iterations. The MIS process work without cache and restart for every datapoint with 10-step warm up. We use 1000 randomly selected images as labeled data, and use the other 600k images as unlabeled data. The accuracy is evaluated on the test set with 26k images. The algorithm runs $1.2*10^5$ iterations, and for every iteration parameters are updated by 50 labeled data and 500 unlabeled data. For first 1000 iterations, only supervised learning is executed.

\begin{figure}
\vskip 0.2in
\begin{center}
\centerline{\includegraphics[width=\columnwidth]{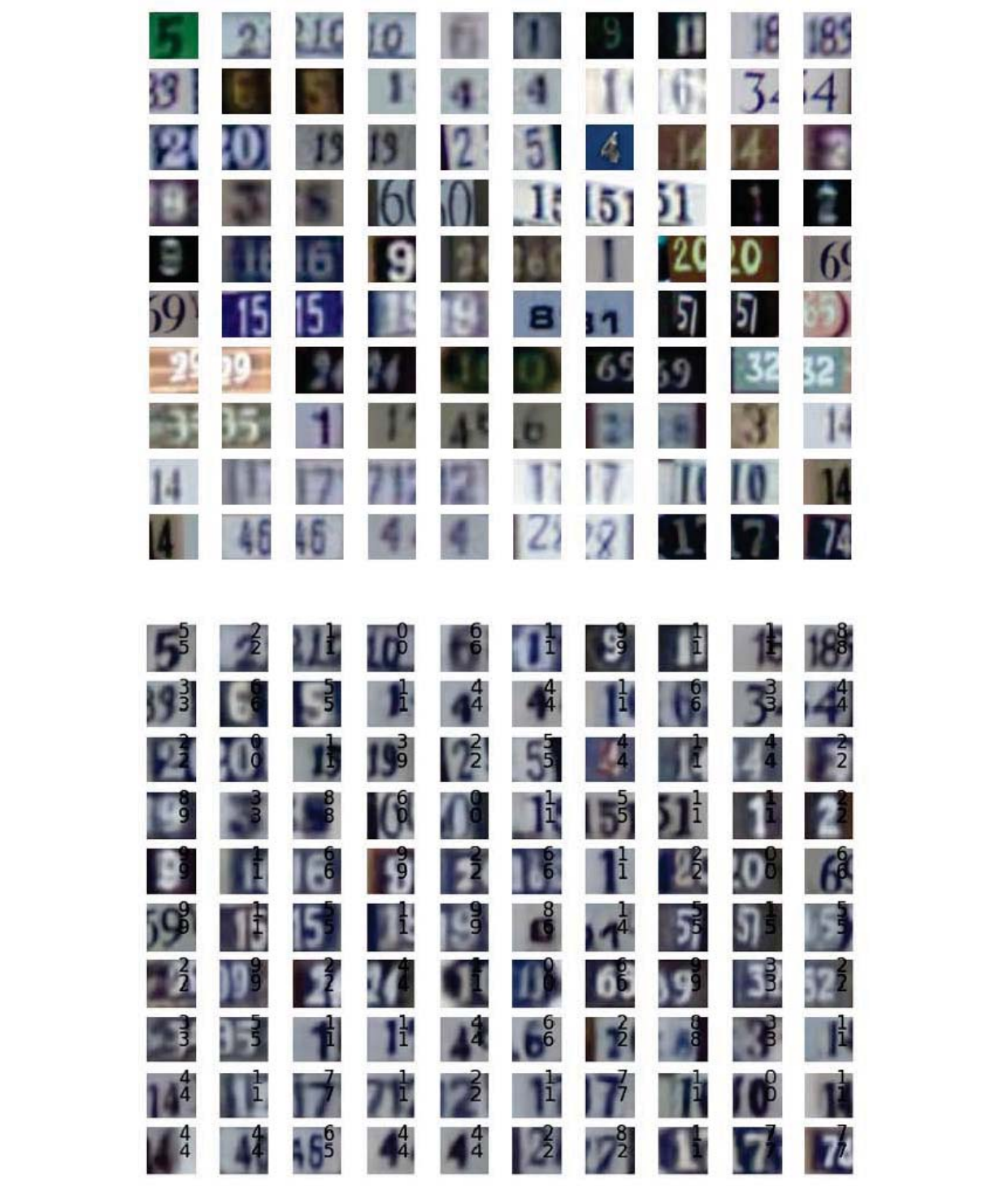}}
\caption{Reconstructions (lower) for the images in the SVHN test set (upper), with 220d Bernoulli $p(h)$. 
The two digits at the right of the lower picture are the true labels and predicted labels respectively.}
\label{fg:svhn_rec}
\end{center}
\vskip -0.2in
\end{figure}

\begin{figure}
\vskip 0.2in
\begin{center}
\centerline{\includegraphics[width=\columnwidth]{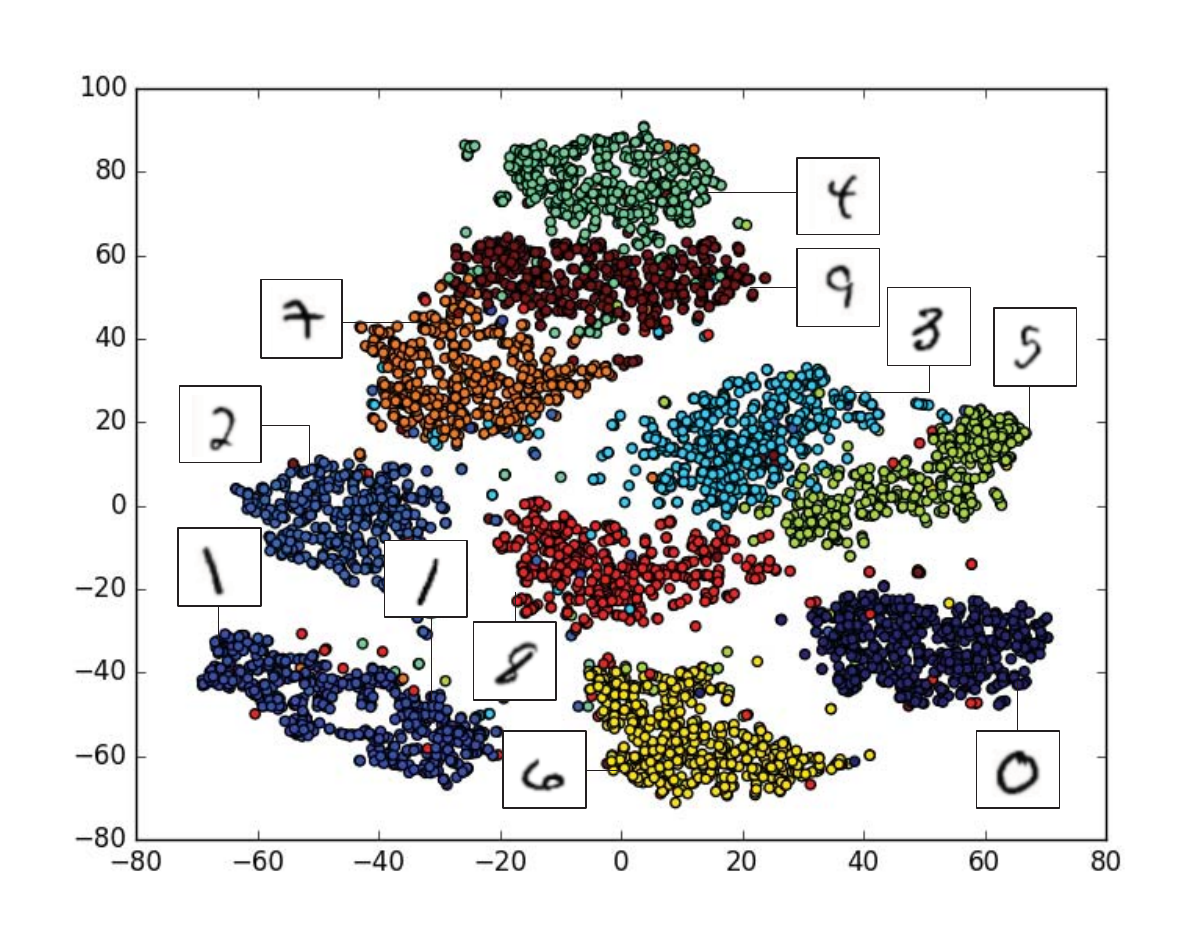}}
\caption{t-SNE of 60d latent codes inferred by unsupervised JAE with 60d Bernoulli $p(h)$. Different colors represent different class variables in MNIST dataset. 
}
\label{fg:tsne}
\end{center}
\vskip -0.2in
\end{figure}

\begin{figure}
\vskip 0.2in
\begin{center}
\centerline{\includegraphics[width=\columnwidth]{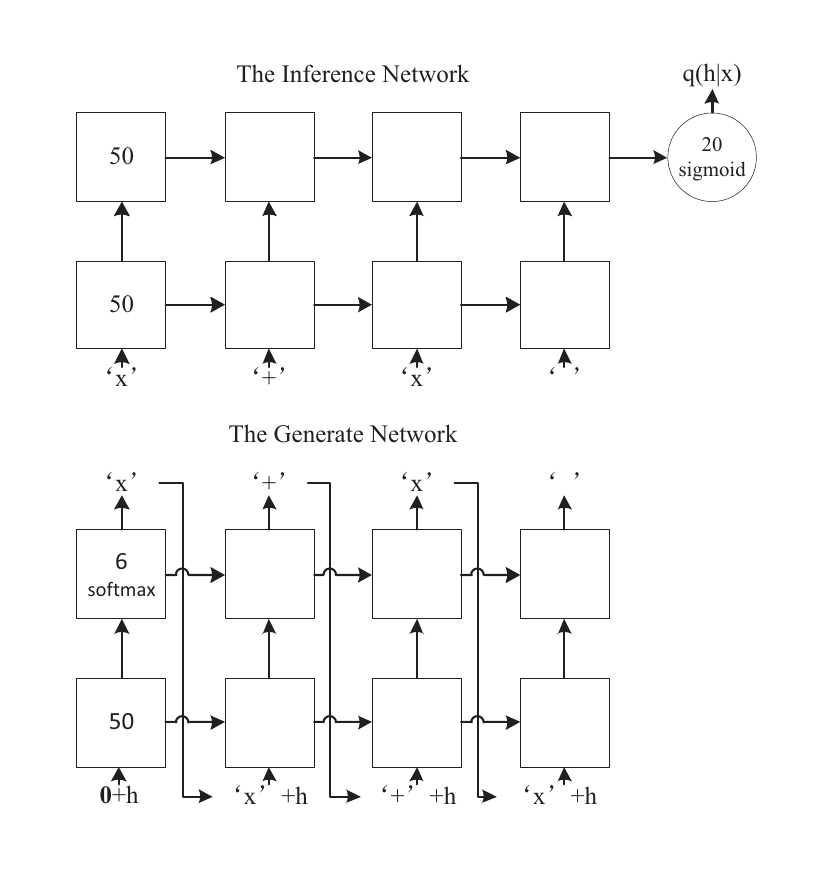}}
\caption{Architecture of the JAE model on context free grammar dataset.}
\label{fg:graph_lstm}
\end{center}
\vskip -0.2in
\end{figure}

\begin{figure}
\vskip 0.2in
\begin{center}
\centerline{\includegraphics[width=\columnwidth]{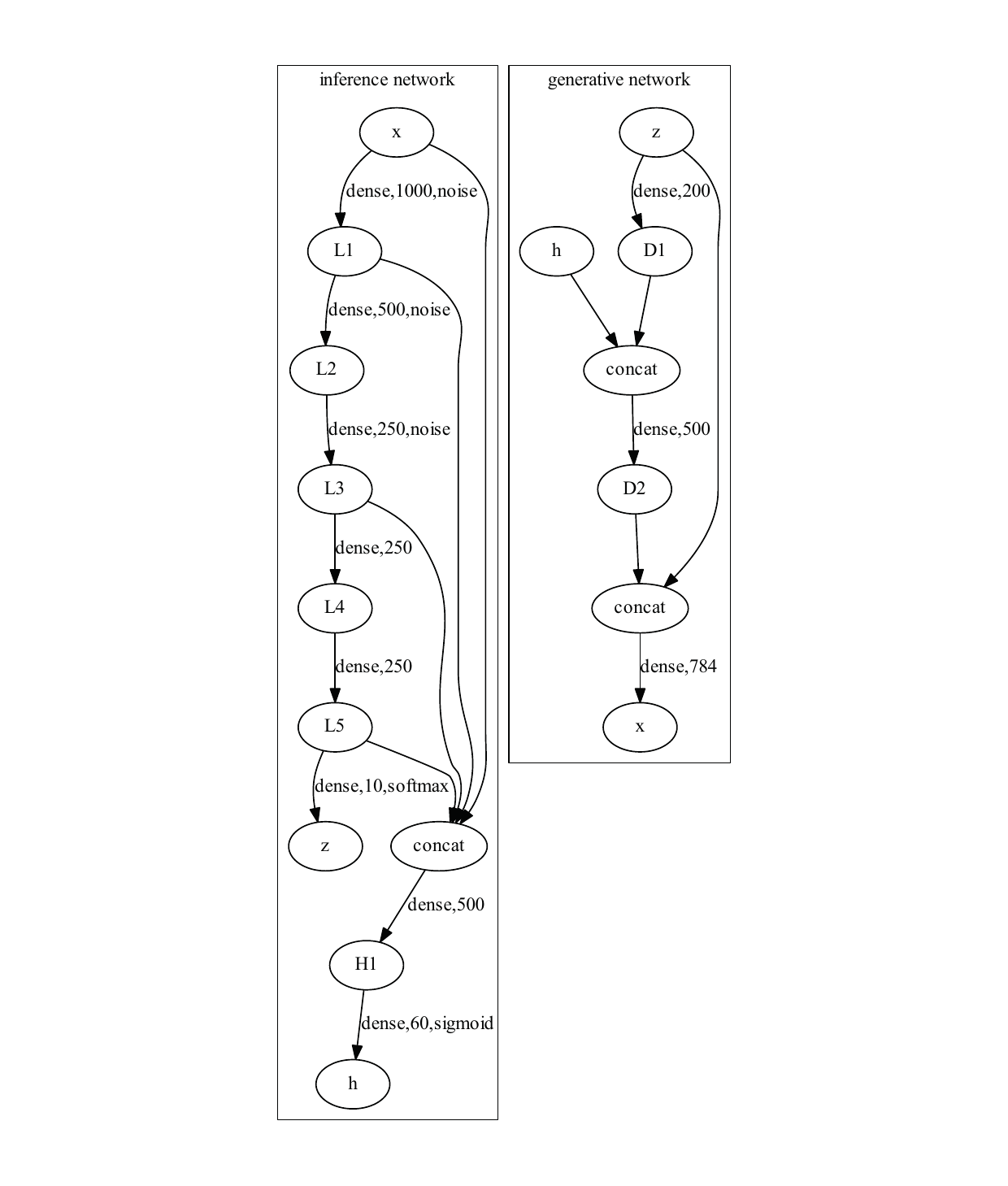}}
\caption{Architecture of the semi-JAE model on MNIST. }
\label{fg:graph_mnist}
\end{center}
\vskip -0.2in
\end{figure}

\begin{figure}
\vskip 0.2in
\begin{center}
\centerline{\includegraphics[width=\columnwidth]{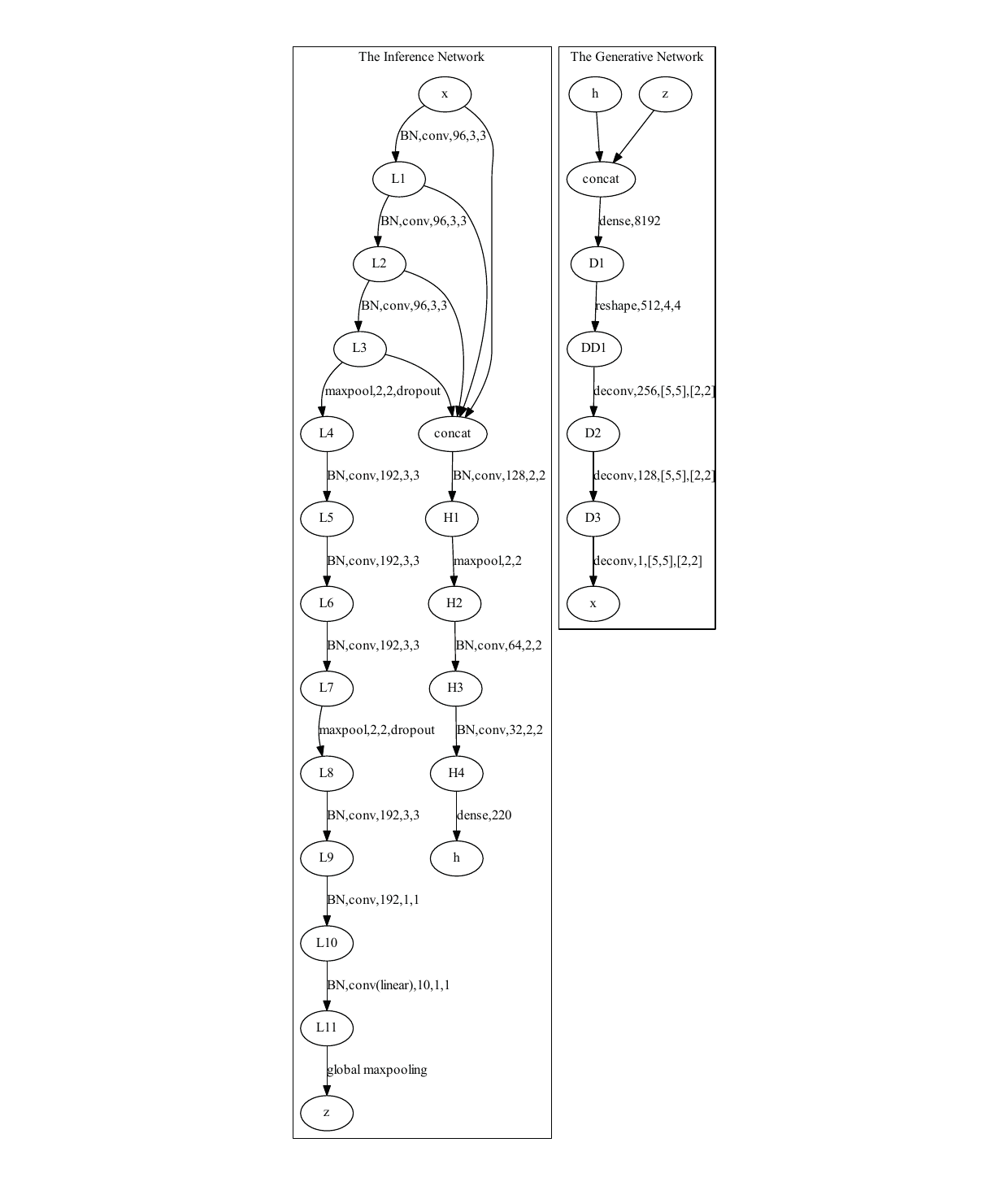}}
\caption{Architecture of the semi-JAE model on SVHN.}
\label{fg:graph_svhn}
\end{center}
\vskip -0.2in
\end{figure}

\end{document}